\documentclass[sigconf]{acmart}
\usepackage{subcaption}
\usepackage{amsmath}

\renewcommand\footnotetextcopyrightpermission[1]{} 
\DeclareMathOperator*{\argmin}{argmin}

\usepackage{booktabs} 

\acmConference[KDD '18]{ACM Woodstock conference}{19--23 August 2018}{London, United Kingdom} 
\acmYear{2018}
\copyrightyear{2018}

\begin{document}
\title{Variable Selection and Task Grouping for Multi-Task Learning}

  


\author{Jun-Yong Jeong}
\orcid{0000-0001-8208-4428}
\affiliation{%
  \institution{POSTECH}
  \city{Pohang} 
  \state{South Korea} 
}
\email{june0227@postech.ac.kr}

\author{Chi-Hyuck Jun}
\affiliation{%
  \institution{POSTECH}
  \city{Pohang} 
  \state{South Korea} 
}
\email{chjun@postech.ac.kr}


\begin{abstract}
We consider multi-task learning, which simultaneously learns related prediction tasks, to improve generalization performance. 
We factorize a coefficient matrix as the product of two matrices based on a low-rank assumption.
These matrices have sparsities to simultaneously perform variable selection and learn and overlapping group structure among the tasks.
The resulting bi-convex objective function is minimized by alternating optimization, where sub-problems are solved using alternating direction method of multipliers and accelerated proximal gradient descent. 
Moreover, we provide the performance bound of the proposed method. 
The effectiveness of the proposed method is validated for both synthetic and real-world datasets.
%
\end{abstract}

%
%



\maketitle
%
\section{Introduction}
Multi-task learning (MTL) refers to simultaneously learning multiple related prediction tasks rather than learning each task independently \cite{MTL:Caruana1997,MTL:survey17}.
Simultaneous learning enables us to share common information among related tasks, and works as an inductive bias to improve generalization performance.
MTL is based on the premise the fact that humans can learn a new task easily when they already have knowledge from similar tasks. 

The major challenges in MTL are how to share common information among related tasks and how to prevent unrelated tasks from being sharing. 
Previous studies achieved this by performing variable selection \cite{MTL:variable_dirty_Jalali2010,MTL:MMTFL_Wang2016}, assuming a low-rank structure \cite{MTL:low-rank_Ando2005,MTL:low_rank:Argyriou2008,MTL:low_rank_group_Kumar2012}, or learning structures among tasks \cite{MTL:grouping_Jacob2008,MTL:groping_TPAMI_Zhou2016,MTL:asmtl_Giwoong2016,MTL:TAT_Han2015}. 

The variable selection approach selects a subset of variables for related tasks \cite{MTL:variable_infnorm_Liu2009,MTL:MMTFL_Wang2016}. 
Traditional studies are based on a strict assumption that selected variables are shared among all tasks \cite{MTL:variable_infnorm_Liu2009,MTL:Obozinski2010}.
Recent studies have suggested a more flexible approach that involves selecting variables by decomposing a coefficient into a shared part and an individual part \cite{MTL:variable_dirty_Jalali2010,MTL:variable_pdirty_Lobato2015} or factorizing a coefficient using a variable specific part and a task-variable part \cite{MTL:MMTFL_Wang2016}. 
Although the variable selection approach provides better interpretability than the other approaches, it has limited ability to share common information among related tasks. 

The low-rank approach assumes that coefficient vectors lie within a low-dimensional latent space \cite{MTL:low-rank_Ando2005,MTL:low_rank:Argyriou2008} and is a representation learning that transform input variables into low-dimensional features and learn coefficient vectors in the feature space \cite{MTL:low_rank_bound_Maurer2016}.
The low-rank approach has also been widely studied in multi-output regression, where entire tasks have real-valued outputs and share the same training set \cite{MOR:JASA_Chen2012}.
It can be achieved by imposing a trace-constraint \cite{MTL:low_rank:Argyriou2008}, encouraging sparsity on the singular values of a coefficient matrix \cite{MTL:low_rank_l1_trace_Richard2012,MTL:low_rank_l1_trace_Richard2013,MTL:low-rank_spectral_ksupport_McDonald2016,MTL:Low-rank_Han2016}, or factorizing a coefficient matrix as the product of a variable-latent matrix and a latent-task matrix \cite{MTL:low-rank_Ando2005,MTL:low_rank:Argyriou2008,MTL:Low_rank_Kang2011,MTL:low_rank_group_Kumar2012,MTL:low_rank_bound_Maurer2016}.
Several studies have shown that the low-rank approach is equivalent to an approach that assumes a group structure among tasks \cite{MTL:low-rank_spectral_ksupport_McDonald2016,MTL:grouping_Zhou2011}.
Thus, recent studies on the low-rank approach have focused on improving the ability of models to learn group structures among tasks \cite{MTL:low_rank_grouping_Kang2011,MTL:low_rank_group_Kumar2012,MTL:low-rank_group_barzilai2015}. 
The low-rank approach provides a flexible way to share common information among related tasks and reduces the effective number of parameters.

It attempts to combine the variable selection approach and the learning of group structures among tasks, especially those based on the low-rank approach.
This combination learns sparse representations to provide better interpretability and shares common information among related tasks in a group to improve generalization performance. 
Previous studies have either partially achieved this goal or have limitations. 
For example, Chen and Huang \cite{MOR:JASA_Chen2012} factorized a coefficient matrix and imposed sparsity between the rows of a variable-latent matrix to perform variable selection. 
They solved multi-output regression and did not explicitly learn a group structure among tasks.
Kumar and Daum{\'e} III \cite{MTL:low_rank_group_Kumar2012} also factorized a coefficient matrix and imposed sparsity within the column vectors of a latent-task matrix to learn overlapping group structures among tasks, but they did not perform variable selection. 
Richard et al. \cite{MTL:low_rank_l1_trace_Richard2012,MTL:low_rank_l1_trace_Richard2013} penalized both a trace norm and an $\ell_1$ norm to simultaneously perform variable selection and impose a low-rank structure. However, a trace norm penalty requires the use of extensive assumptions to ensure a low-rank structure \cite{MTL:low-rank_trace_Mishra2013} and singular value decomposition for each iteration of the optimization.
Han and Zhang \cite{MTL:TAT_Han2015} learned overlapping group structures among tasks by decomposing a coefficient matrix into component matrices, but they could not remove irrelevant variables.
Wang et al. \cite{MTL:MMTFL_Wang2016} factorized a coefficient matrix as the product of full-rank matrices to perform variable selection, but did not
explicitly learn a group structure among tasks.

This paper proposes the variable selection and task grouping-MTL (\textbf{VSTG-MTL}) approach, which simultaneously 
performs variable selection and learns an overlapping group structure among tasks based on the low-rank approach.
Our main ideas are to express a coefficient matrix as the product of a variable-latent matrix and a latent-task matrix and impose sparsities on these matrices. 
The sparsities between and within the rows of a variable-latent matrix help the model to select relevant variables and have flexibility. 
We also encourage sparsity within the columns of a latent-task matrix to learn an overlapping group structure among tasks, and note that learning the latent-task matrix is equivalent to learning task coefficient vectors in a feature space where features can be highly correlated.
This correlation is considered in the model by applying a $k$-support norm \cite{MTL:low-rank_spectral_ksupport_McDonald2016}. 
The resulting bi-convex problem is minimized by alternating optimization, where sub-problems are solved by applying the alternating direction method of multipliers (ADMM) and accelerated proximal gradient descent.
We provide an upper bound on the excess risk of the proposed method to guarantee its performance.  
Experiments conducted on four synthetic datasets and five real-world datasets show that the proposed VSTG-MTL approach outperforms several benchmark MTL methods and that the $k$-support norm is effective on handling the possible correlation. 

We summarize our contributions as follows
\begin{itemize}
\item To the best our knowledge, this is the first work that simultaneously performs variable selection and learns an overlapping group structure among tasks using the low-rank approach. 
\item 
We focus on the possible correlation from a representation learning and apply a $k$ support norm to improve generalization performance. 
\item We present an upper bound on the excess risk of the proposed method.
\end{itemize}
\section{Preliminary}
In this section, we explain multi-task learning, low-rank structures, and $k$-support norms.
\subsection{Low-rank Structure for Multi-task Learning}
Suppose that we are given $D$ variables and $T$ supervised learning tasks, where the $j$-th task has an input matrix $\textbf{X}_j=\bigg[ \left(\textbf{x}_j^1\right)^T,\ldots$, \\ $\left(\textbf{x}_j^{N_j}\right)^T \bigg]^T \in \mathbb{R}^{N_j \times D}$ with $\textbf{x}_j^n \in \mathbb{R}^D$ and an output vector $\textbf{y}_j=\left[y_j^1,\ldots,y_j^{N_j} \right]^T \in \mathbb{R}^{N_j}$.
Next, we focus on a linear relation between input and output 
\begin{equation} \label{eq:lin_mdl}
y_j^n=f(\textbf{w}_j^T\textbf{x}_j^n),
\end{equation}
where $f$ is an identity function for a regression problem $y_j^n\in \mathbb{R}$ or a logit function for a binary classification problem $y_j^n \in \{-1,1\}$ and $\textbf{w}_j \in \mathbb{R}^{D}$ represents a coefficient vector for the $j$-th task. Then, we can describe the matrix $\textbf{W}=[\textbf{w}_1,\ldots,\textbf{w}_T]$ as a coefficient matrix. 

We then impose a low-rank structure on the coefficient matrix \textbf{W} to share common information among related tasks \cite{MTL:low-rank_Ando2005,MTL:low_rank_group_Kumar2012}.
The low-rank structure assumes that the coefficient vectors $\textbf{w}_j$, $j=1,\ldots,T$ lie within a low-dimensional latent space and are expressed by a linear combination of latent bases.
The coefficient matrix \textbf{W} can be factorized as the product of two low rank matrices \textbf{U}\textbf{V},
where $\textbf{U} \in \mathbb{R}^{D\times K}$ is the variable-latent matrix, $\textbf{V} \in \mathbb{R}^{K\times T}$ is the latent-task matrix, and $K << \min\{D,T\}$ is the number of latent bases.
Then, we can express
the coefficient of the $i$-th variable for the $j$-th task $w_{ij}$ and the coefficient vector for the $j$-th task $\textbf{w}_j$ as follows
\begin{gather}
\label{eq:coefficient_factorization}
w_{ij} = \textbf{u}^i\textbf{v}_j
\\
\textbf{w}_{j}=\textbf{U}\textbf{v}_j = \sum_{r=1}^K v_{rj}\textbf{u}_r,
\end{gather}
where $\textbf{u}^i \in \mathbb{R}^{1\times K}$ and $\textbf{u}_r \in \mathbb{R}^D$ are the $i$-th row vector and $r$-th column vector of the variable-latent matrix $\textbf{U}$, respectively,
and $\textbf{v}_j\in \mathbb{R}^K$ is the $j$-th column vector of the variable-latent matrix \textbf{V}. 
The above equations reveal the roles of the two matrices.
The $i$-th row vector $\textbf{u}_i$ and the $r$-th column vector $\textbf{u}^r$ can be regarded as being of equal importance of that of the $i$-th variable and $r$-th latent basis. 
Then, the $j$-th column vector $\textbf{v}_j$ can be regarded as the weighting vector for the $j$-th task.

Furthermore, this low-rank structure can be considered as a representation learning \cite{MOR:JASA_Chen2012,MTL:low_rank_bound_Maurer2016}. Thus, we can rewrite Eq. \eqref{eq:lin_mdl} as
\begin{equation} \label{eq:feature_extraction}
y_j^n=f\left(\textbf{w}_j^T\textbf{x}_j^n\right)=f\left(\textbf{v}_j^T(\textbf{U}^T\textbf{x}_j^n)\right).
\end{equation}
The transpose of the variable-latent matrix $\textbf{U}^T$ and the $j$-th weighting vector $\textbf{v}_j$ represent a linear map from a variable space to a feature space, where $\textbf{x}\in \mathbb{R}^D$ is mapped to $[\textbf{u}_1^T\textbf{x},\ldots,\textbf{u}_K^T\textbf{x}]^T$ $\in \mathbb{R}^K$ and the coefficient vector of the $j$-th task on the feature space, respectively.
We note that unless the latent bases $\textbf{u}_r$, $r=1,\ldots,K$ are orthogonal, the features $\textbf{u}_r^T\textbf{x}, r=1,\ldots,K$ can be highly correlated. 
\subsection{The $k$-support Norm}
We commonly use an $\ell_1$ norm as a convex approximation to an $\ell_0$ norm in regularized regression.
When features are correlated and form several groups, the $\ell_1$ norm penalty tends to select a few features from the groups, where we can improve the generalization performance by selecting all correlated features \cite{ksupport_nips2012}.
In this case, a possible alternative to the $\ell_1$ norm is a $k$-support norm $\|\cdot\| _k^{sp}$, i.e., the tightest convex relaxation of sparsity within a Euclidean ball \cite{ksupport_nips2012}.
The $k$-support norm is defined for each  $\textbf{w}\in\mathbb{R}^{K}$ as follows:
\begin{displaymath}
\|\textbf{w}\|_k^{sp} := \min \left \{ \sum_{g\in \mathcal{G}_k}\|\textbf{s}_g\|:\text{supp}(\textbf{s}_g)\subseteq g,\sum_{g\in\mathcal{G}_k}\textbf{s}_g=\textbf{w}\right \},
\end{displaymath}
where $\mathcal{G}_k$ denotes all subsets of ${1,\ldots,K}$ of cardinality of at most $k$. 
Moreover, $\|\textbf{w}\|_1^{sp}=\|\textbf{w}\|_1$ and $\|\textbf{w}\|_K^{sp}=\|\textbf{w}\|_2$. 
Thus, the $k$-support norm is a trade-off between an $\ell_1$ norm and an $\ell_2$ norm. 
This property can be enhanced by inspecting the following proposition:
\begin{proposition}
\normalfont
(\textit{Proposition 2.1} \cite{ksupport_nips2012}) \textit{For every} \textbf{w} $\in \mathbb{R}^K$,
\begin{displaymath}
\|\textbf{w}\|_k^{sp} = \left(\sum_{l=1}^{k-p-1}\left(|w|_l^{\downarrow}\right)^2
+\frac{1}{p+1} \left(\sum_{l=k-p}^K |w|_l^{\downarrow}\right)^2 \right)^{\frac{1}{2}},
\end{displaymath}
\textit{where} $w_l^{\downarrow}$ \textit{is the} $l$\textit{-th largest element of the absolute values of} \textbf{w}, \textit{letting} $|w|_0^{\downarrow}$ \textit{denote} $+\infty$, \textit{and} $p$ \textit{is the unique integer in} $\left\{ 0,\ldots,k-1\right\}$ \textit{satisfying}
\begin{displaymath}
|w|_{k-p-1}^{\downarrow} > \frac{1}{p+1} \sum_{l=k-p}^{K} |w|_l^{\downarrow} >= |w|_{k-p}^{\downarrow}.
\end{displaymath}
\end{proposition}
The above proposition shows that the $k$-support norm imposes both the uniform shrinkage of an $\ell_2$ norm on the largest components and the spare shrinkage of an $\ell_1$ norm on the smallest components. 
Thus, in a similar way to Elastic net \cite{elastic_net}, the $k$-support norm penalty encourages the selection of a few groups of correlated features and imposes the uniform shrinkage of the $\ell_2$ norm on the selected groups.

\section{Formulation}
\begin{figure} [t]
\label{Fig1.proposed_example}
\centering
\begin{subfigure}[t] {0.15\textwidth}
\includegraphics[width=\textwidth, height= 1.6in]{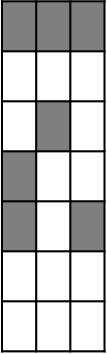} \caption{A variable-latent matrix \textbf{U}}
\end{subfigure}   
\quad
\begin{subfigure}[t] {0.2\textwidth}
\raisebox{.2\textwidth}
{               \includegraphics[width=\textwidth,height=0.6in]{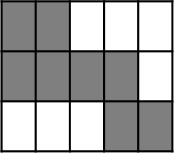}
}
\caption{A latent-task matrix \textbf{V}}
\end{subfigure} 
\\ 
\begin{subfigure}[t] {0.2\textwidth}
\includegraphics[width=\textwidth, height= 1.6in]{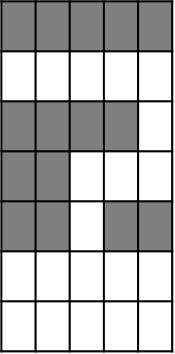}   
\caption{A coefficient matrix \textbf{W}}
\end{subfigure} 


\caption{Example of VSTG-MTL. The gray and white entries express non-zero and zero values, respectively. The feature-latent matrix \textbf{U} shows the sparsities between and within its rows representing variables, and the latent-task matrix \textbf{V} shows the sparsity within its columns representing tasks. The columns of the variable-latent matrix \textbf{U} and the rows of the latent-variable matrix represent latent bases or features}         	
\end{figure}
We aim to simultaneously learn an overlapping group structure among tasks and select relevant variables.
To achieve these goals, we employ the low-rank assumption shown in Section 3.1 and impose sparsities on a variable-latent matrix \textbf{U} and a latent-task matrix \textbf{V}.
Fig. 1 shows an example of VSTG-ML, where the gray and white entries express non-zero and zero values, respectively.
Each row of the variable-latent matrix \textbf{U} and the coefficient matrix \textbf{W} represent a variable. Similarly, each column of the latent-task matrix \textbf{V} and the coefficient matrix \textbf{W} represent a task; 
each column of the variable-latent matrix \textbf{U} and row of the latent-task matrix \textbf{V} represent a latent basis or feature. 
The variable-latent matrix \textbf{U} in Fig. 1(a) shows the sparsities between and within its rows, while the latent-task matrix \textbf{V} in Fig. 1(b) shows the sparsity within its columns. The coefficient matrix \textbf{W} in Fig. 1(c) expresses the product of these matrices. 

The sparsity between the variable importance vectors $\textbf{u}^i$, $i=1,\ldots,D$ induces a model that can be used to select relevant variables \cite{MOR:JASA_Chen2012}. 
If the $i$-th variable importance vector $\textbf{u}^i$ is set to \textbf{0}, then the corresponding variable is removed from the model in accordance with Eq. (2).
For example, in Fig. 1(a), the 2nd, 6th and 7th variables are excluded from the model, whereas the 1st, 3rd, 4th, and 5th variables are selected.
Simultaneously, the sparsity within the variable importance vector $\textbf{u}^i$ improves the flexibility of the model. The latent basis vector $\textbf{u}^r$ does not necessarily depend on all selected variables.
Instead, it can have non-zero values from a subset of the selected variables.

The sparsities within the weighting vectors $\textbf{v}_j$, $j=1,\ldots,T$ learn an overlapping group structure among tasks. 
\cite{MTL:low_rank_group_Kumar2012}.
The group structure among tasks are decided by the sparsity patterns on the weighting vector $\textbf{v}_j$.
Tasks with same sparsity patterns on the weighting vector $\textbf{v}_j$ belong to the same group, whereas those with the orthogonal ones belong to disjoint groups. 
Two groups are regard as being overlapped if their sparsity patterns are not orthogonal, i.e., they partially share the latent bases. 
For example, in Fig. 1(b), the 1st and 2nd tasks belong to the same group and share the 2nd latent basis with the 3rd and 4th tasks. However, they do not share any latent basis with the 5th task.
As mentioned in Sec 2.1, learning the $j$-th weighting vectors $\textbf{v}_j$ is equivalent to learning the coefficient vector of the $j$-th task in a feature space induced by the transpose of the variable-latent matrix $\textbf{U}^T$.
The features $\textbf{u}_r^T\textbf{x}$, $r=1,\ldots,K$ can be highly correlated unless the latent bases are orthogonal. 
Thus, instead of the $\ell_1$ norm, the $k$-support norm is appropriate to encouraging the sparsity within the weighting vector $\textbf{v}_j$.
The $k$-support norm induces the less sparse weighting vector $\textbf{v}_j$ than that from the $\ell_1$ norm and similarly enhances the overlaps in the task groups. 

We formulate the following problem
\begin{gather}
\label{pro:consttrained}
\min_{\textbf{U},\textbf{V}}\sum_{j=1}^T \frac{1}{N_j}L\left( \textbf{y}_j,\textbf{X}_j\textbf{U}\textbf{v}_j\right) 
\nonumber \\
\text{s.t } 
\|\textbf{U} \|_1 \leq \alpha_1,
\quad 
\|\textbf{U}\|_{1,\infty} \leq \alpha_2,
\nonumber \\
\sum_{j=1}^T \left(\|\textbf{v}_j\|_{k}^{sp}\right)^2 \leq \beta
\end{gather}
where $L(\cdot,\cdot)$ is the empirical loss function, which becomes a squared loss $\frac{1}{2}\|\textbf{y}_j-\textbf{X}_j\textbf{U} \textbf{v}_j\|_2^2$  for a regression problem and
a logistic loss $\sum_{j=1}^{N_j} \log \left( 1+\exp\left(-{y_{j}^n\textbf{v}_j^T\textbf{U}^T\textbf{x}_j^n}\right) \right)$ for a binary classification problem;
$\|\textbf{U}\|_1= \sum_{i=1}^D\sum_{r=1}^K|u_{ir}|$ is the $\ell_1$ norm;
$\|\textbf{U}\|_{1,\infty}= \sum_{i=1}^D\|\textbf{u}^i\|_{\infty}$ is the $\ell_{1,\infty}$ norm;
$\|\textbf{v}_j \|_k^{sp}$ is the $k$-support norm; and $\alpha_1,\alpha_2,$ and $\beta$ are the constraint parameters. 
The $\ell_{1,\infty}$ norm and the $\ell_{1}$ norm constraints encourage the sparsities between and within the variable importance vectors $\textbf{u}^i$, $i=1,\ldots,D$.
The squared $k$-support norm constraint encourages the sparsity within the weighting vectors $\textbf{v}_j$, $j=1,\ldots,K$ while considering possible correlations among the features.
\section{Optimization}
The optimization problem (5) is bi-convex for the variable-latent matrix $\textbf{U}$ and latent-task matrix $\textbf{V}$; for a given $\textbf{U}$, it is convex for $\textbf{V}$ and vice versa.
We transform the above constraint problem to the following regularized objective function
\begin{equation} \label{eq:objective}
\sum_{j=1}^T \frac{1}{N_j}L\left(\textbf{y}_j, \textbf{X}_j\textbf{U}\textbf{v}_j\right) 
+ \gamma_1\|\textbf{U}\|_1 + \gamma_2 \|\textbf{U}\|_{1,\infty} + \mu\sum_{j=1}^T \left( \|\textbf{v}_j \|_k^{sp}\right)^2,
\end{equation}
where $\gamma_1, \gamma_2$, and $\mu$ are the regularization parameters.
Then, we apply alternating optimization to obtain the partial minimum of the objective function \eqref{eq:objective}.

Initial estimates of the matrices \textbf{U} and \textbf{V} are crucial in generalization performance considering that the optimization function \eqref{eq:objective} is non-convex.
To compute reasonable initial estimates, for each task, we learn a ridge regression or logistic regression coefficient:
\begin{equation} \label{eq:w_init_vector}
\textbf{w}_j^{init}:= \argmin_{\textbf{w}}\frac{1}{N_j} L\left(\textbf{y}_j,\textbf{X}_j\textbf{w} \right) + \left(\sqrt{\gamma_1^2+\gamma_2^2+\mu^2}\right) \|\textbf{w} \|_2^2.
\end{equation}
We also define an initial coefficient matrix that stacks the ridge coefficient as a column vector:
\begin{equation}
\label{eq:W_init_matrix}
\textbf{W}^{init}=[\textbf{w}_1^{init},\ldots,\textbf{w}_T^{init}].
\end{equation}
Then, we compute the top-$K$ left-singular vectors $\textbf{P}\in \mathbb{R}^{D\times K}$, the top-$K$ right singular vectors $\textbf{Q}\in \mathbb{R}^{T\times K}$, and the top-$K$ singular value matrix $\mathbf{\Sigma} \in \mathbb{R}^{K \times K}$ of the initial coefficient matrix $\textbf{W}^{init}$.
The initial estimates $\textbf{U}^{init}$ and $\textbf{V}^{init}$ are given by $\textbf{P}\mathbf{\Sigma}^{\frac{1}{2}}$ and $\mathbf{\Sigma}^{\frac{1}{2}} \textbf{Q}^T$, respectively. 
\subsection{Updating U} 
For a fixed latent-task matrix \textbf{V}, the objective function for the variable-latent matrix \textbf{U} becomes as follows:
\begin{equation} \label{eq:obj_U_origin}
\sum_{j=1}^T \frac{1}{N_j}L\left(\textbf{y}_j,\textbf{X}_j\textbf{U}\textbf{v}_j\right)+ \gamma_1\|\textbf{U}\|_1 + \gamma_2 \|\textbf{U}\|_{1,\infty}.
\end{equation}
It is solved by applying ADMM \cite{ADMM_Boyd_2011}.
First, we introduce auxiliary variables $\textbf{Z}_h\in \mathbb{R}^{D\times K},\textbf{ }h=1,2,3$ and reformulate the above problem as follows:
\begin{gather}
\label{eq:U_admm_aggregate}
\min_{\textbf{U},\textbf{Z}_1, \textbf{Z}_2, \textbf{Z}_3} \sum_{j=1}^T \frac{1}{N_j}L\left(\textbf{y}_j,\textbf{X}_j\textbf{Z}_1\textbf{v}_j\right) + \gamma_1\|\textbf{Z}_2\|_1 + \gamma_2 \|\textbf{Z}_3\|_{1,\infty} \nonumber \\
\text{s.t } \textbf{AU} + \textbf{BZ} = \textbf{0},
\end{gather}
where 
\begin{displaymath}
\textbf{A} = \begin{bmatrix}
 \textbf{I}_D \\ \textbf{I}_D \\ \textbf{I}_D
\end{bmatrix},
\quad
\textbf{B} = \begin{bmatrix}
-\textbf{I}_D  & \textbf{0} &\textbf{0}\\ \textbf{0} & -\textbf{I}_D & \textbf{0}\\ \textbf{0}& \textbf{0}& -\textbf{I}_D \\
\end{bmatrix}
, \quad \text{and} \quad
\textbf{Z} = \begin{bmatrix}
\textbf{Z}_1 \\ \textbf{Z}_2 \\ \textbf{Z}_3
\end{bmatrix}.
\end{displaymath}
Let $\mathbf{\Lambda}_h$ be a scaled Lagrangian multiplier for the $h$-th auxiliary variables $\textbf{Z}_h$ and $\mathbf{\Lambda} = \left[ \mathbf{\Lambda}_1^T, \mathbf{\Lambda}_2^T, \mathbf{\Lambda}_3^T \right]^T$. 
Then, the variable-latent matrix \textbf{U} is updated as follows:
\begin{equation} \label{eq:update_U}
\begin{aligned}
\textbf{U}^{t+1}&:
=\argmin_{\textbf{U}}
\frac{\rho}{2} \|\textbf{AU} + \textbf{B}\textbf{Z}^t + \mathbf{\Lambda}^t\|_F^2 \\
&= \argmin_{\textbf{U}} \|\textbf{U} - \textbf{Z}_1^t + \mathbf{\Lambda}_1^t \|_F^2 + \|\textbf{U} - \textbf{Z}_2^t + \mathbf{\Lambda}_2^t \|_F^2 \\
&+ \|\textbf{U} - \textbf{Z}_3^t + \mathbf{\Lambda}_3^t \|_F^2\\
&=\frac{1}{3}\sum_{h=1}^{3}\left(\textbf{Z}_h^t - \mathbf{\Lambda}_h^t \right),
\end{aligned}
\end{equation}
where $t$ denotes the iteration and $\rho>0$ is the ADMM parameter.\\
The auxiliary variables $\textbf{Z}_h$, $h=1,2,3$ are updated by solving the following problem
\begin{equation}
\begin{aligned}
\textbf{Z}^{t+1}&:=\argmin_{\textbf{Z}} \sum_{j=1}^T \frac{1}{N_j}L\left(\textbf{y}_j,\textbf{X}_j\textbf{Z}_1\textbf{v}_j\right) + \gamma_1\|\textbf{Z}_2\|_1 + \gamma_2 \|\textbf{Z}_3\|_{1,\infty}\\
&+\frac{\rho}{2}\| \textbf{A}\textbf{U}^{t+1} +\textbf{B}\textbf{Z} + \mathbf{\Lambda}^t\|_F^2.
\end{aligned}
\end{equation}
In detail, the first auxiliary variable $\textbf{Z}_1$ is updated as follows
\begin{equation} \label{eq:update_Z1}
\textbf{Z}_1^{t+1}:=\argmin_{\textbf{Z}_1} \sum_{j=1}^T\frac{1}{N_j}L\left(\textbf{y}_j, \textbf{X}_j\textbf{Z}_1\textbf{v}_j\right) +
\frac{\rho}{2} \|\textbf{U}^{t+1} -\textbf{Z}_1+\mathbf{\Lambda}_1^t \|_F^2.
\end{equation}
For regression problems with a squared loss, we can compute the close-form updating equation by equating the gradient of the optimization problem \eqref{eq:update_Z1} to zero as follows:
\begin{gather*} \label{eq:update_Z1_regress}
\sum_{j=1}^T\frac{1}{N_j}\textbf{X}_j^T\textbf{X}_j\textbf{Z}_1\textbf{v}_j\textbf{v}_j^T + \rho\textbf{Z}_1 =
\sum_{j=1}^T \frac{1}{N_j} \textbf{X}_j^T\textbf{y}_j\textbf{v}_j^T + \rho\left(\textbf{U}^{t+1} + \mathbf{\Lambda^{t}} \right)
\nonumber \\
\left[\sum_{j=1}^T\frac{1}{N_j}\textbf{v}_j\textbf{v}_j^T \otimes \textbf{X}_j^T\textbf{X}_j^T +\rho\textbf{I} \right]\text{vec}(\textbf{Z}_1) \nonumber \\
=\sum_{j=1}^T\frac{1}{N_j}\text{vec}\left(\textbf{X}_j^T\textbf{y}_j\textbf{v}_j^T \right) + \rho\text{vec}\left(\textbf{U}^{t+1} + \mathbf{\Lambda}_1^t \right),
\end{gather*}
where $\text{vec}(\cdot)$ is the vectorization operator and $\otimes$ is the Kronecker product. The above linear system of equations is solved by using the Cholesky or LU decomposition.\\
For binary classification problems with a logistic loss, it is solved by using L-BFGS \cite{Numerical_opt}, where the gradient is given as follows:
\begin{displaymath}
\nabla_{\textbf{Z}_1} = \sum_{j=1}^T\sum_{n=1}^{N_j}\frac{-y_j^n\textbf{x}_j^n\textbf{v}_j^T}{1+\exp{\left( y_j^n\textbf{v}_j^T\textbf{Z}_1^T\textbf{x}_j^n\right)}}
+ \rho \left( \textbf{Z}_1 -\textbf{U}^{t+1}-\mathbf{\Lambda}_1^{t} \right).
\end{displaymath}
\\
The other auxiliary variables  $\textbf{Z}_2$ and  $\textbf{Z}_3$ are updated as follows:
\begin{equation} \label{eq:update_Z2}
\begin{aligned}
\textbf{Z}_2^{t+1}&:=\argmin_{\textbf{Z}_2} \gamma_1\|\textbf{Z}_2\|_1 + \frac{\rho}{2} \|\textbf{U}^{t+1}-\textbf{Z}_2 + \mathbf{\Lambda}_2^t \|_F^2 
\\
&= \text{prox}_{\frac{\gamma_1}{\rho}\| \cdot \|_1}\left( \textbf{U}^{t+1} + \mathbf{\Lambda}_2^t \right),
\end{aligned}
\end{equation}
\begin{equation} \label{eq:update_Z3}
\begin{aligned}
\textbf{Z}_3^{t+1}&:=\argmin_{\textbf{Z}_3} \gamma_2 \|\textbf{Z}_3\|_{1,\infty} + \frac{\rho}{2}\|\textbf{U}^{t+1} -\textbf{Z}_3 + \mathbf{\Lambda}_3^t \|_F^2
\\
&=\text{prox}_{\frac{\gamma_2}{\rho}\|\cdot\|_{1,\infty}}\left(\textbf{U}^{t+1} + \mathbf{\Lambda}_3^t \right),
\end{aligned}
\end{equation}
where $\text{prox}_{\lambda\|\cdot\|_1}(\cdot)$ 
and $\text{prox}_{\lambda\|\cdot\|_{1,\infty}}(\cdot)$ are the proximal operators of an $\ell_1$ norm and an $\ell_{1,\infty}$ norm, respectively, which are shown in \cite{proximal_FTO_2014}.\\
The Lagrangian multipliers $\mathbf{\Lambda}_h$, $h=1,2,3$ are updated as follows:
\begin{equation} \label{eq:update_Lambda}
\mathbf{\Lambda}_h^{t+1}:=\mathbf{\Lambda}_h^{t+1}+\textbf{U}^{t+1} - \textbf{Z}^{t+1}_h.
\end{equation}
Then, the primal and dual residuals $\textbf{r}^{t+1}$ and $\textbf{s}^{t+1}$ are given by
\begin{equation} \label{eq:resi_primal}
\textbf{r}^{t+1}=\textbf{A}\textbf{U}^{t+1} + \textbf{B}\textbf{Z}^{t+1}
\end{equation}
\begin{equation} \label{eq:resi_dual}
\textbf{s}^{t+1} = \rho\textbf{A}^T\textbf{B}\left( \textbf{Z}^{t+1} - \textbf{Z}^t\right)
\end{equation}
Note that the updating equation for the variable-latent matrix \textbf{U} in Eq. \eqref{eq:update_U} does not guarantee sparsity. 
Thus, after convergence, the final variable-latent matrix \textbf{U} is given by the second auxiliary variable $\textbf{Z}_2$, which guarantees sparsity due to the proximal operator of the $\ell_1$ norm.
\subsection{Updating V}
\begin{table}[t]
\label{my-label}
\begin{tabular}{l}
\hline
\textbf{Algorithm1} VSTG-MTL\\ \hline
\textbf{input}\\
$\textbf{X}_j$ and $\textbf{y}_j$: training data for task $j=1,\ldots,T$
\\
$K$: number of latent bases
\\
$\gamma_1,\gamma_2,\mu$: regularization parameters
\\
$k$: parameter for the $k$-support norm
\\
$\rho$: parameter for ADMM
\\ \hline
\textbf{output}
\\
\textbf{U}: variable-latent matrix
\\
\textbf{V}: latent-task matrix
\\
\textbf{W}=\textbf{UV}: Coefficient matrix
\\
\hline
\textbf{procedure}
\\
\textbf{1}. Estimate an initial coefficient matrix $\textbf{W}^{init}$ 
\\ by using Eqs. \eqref{eq:w_init_vector} and \eqref{eq:W_init_matrix}.
\\ \textbf{2}. Compute the top-$K$ left singular vectors \textbf{P}, the top-$K$ right \\singular vectors \textbf{Q}, and the top-$K$ singular value matrix $\mathbf{\Sigma}$
\\ $\textbf{W}^{init} = \textbf{P}\mathbf{\Sigma}\textbf{Q}^T$.
\\
\textbf{3}. Estimate initial estimates for $\textbf{U}^{0}$ and $\textbf{V}^{0}$ as follows: \\
$\textbf{U}^{0} = \textbf{P}\mathbf{\Sigma}^{\frac{1}{2}}$ and $\textbf{V}^{0} = \mathbf{\Sigma}^{\frac{1}{2}}\textbf{Q}$ 
\\
\textbf{4}. \textbf{Repeat step 5 to 13.} 
\\ \textbf{5}. \quad \textbf{Repeat step 6 to 8.}
\\ \textbf{6}. \qquad Update the variable-latent matrix \textbf{U} by using Eq. \eqref{eq:update_U}.
\\ \textbf{7}. \qquad Update the auxiliary variables $\textbf{Z}_h$, $h=1,2,3$ \\
by solving Eqs. \eqref{eq:update_Z1}, \eqref{eq:update_Z2}, and \eqref{eq:update_Z3}.
\\ \textbf{8}. \qquad Update scaled Lagrangian multipliers $\mathbf{\Lambda}_h$, $h=1,2,3$ 
\\by using Eq. \eqref{eq:update_Lambda}.
\\ \textbf{9}. \quad \textbf{until} the Frobeneus norms of $\textbf{r}$ and $\textbf{s}$ in Eqs. \eqref{eq:resi_primal} and \eqref{eq:resi_dual} \\
converge.
\\ \textbf{10}. \quad Set the variable-latent matrix \textbf{U} to be equal\\to the second auxiliary variable $\textbf{Z}_2$.
\\ \textbf{11}. \quad \textbf{for} $j=1,\ldots,T$ \textbf{do}
\\ \textbf{12}. \quad \quad Update the weighting vector $\textbf{v}_j$ by solving Eq. \eqref{eq:update_v}. 
\\ \textbf{13}. \quad \textbf{end for}
\\ \textbf{14}. \textbf{until} the objective function in Eq. \eqref{eq:objective} converges.
\\
\hline
\end{tabular}
\end{table}
For a fixed variable-latent matrix \textbf{U}, the problem for the latent-task matrix \textbf{V} is separable into its column vector $\textbf{v}_j$ as follows:
\begin{equation} \label{eq:update_v}
\textbf{v}_j = \argmin_{\textbf{v}} \frac{1}{N_j}L(\textbf{y}_j, \textbf{X}_j\textbf{U}\textbf{v}) + \mu\left(\|\textbf{v}\|_{k}^{sp}\right)^2
\end{equation}
The $j$-th weighting vector $\textbf{v}_j$ is updated by solving the $k$-support norm regularized regression or logistic regression, where a input matrix becomes $\textbf{X}_j\textbf{U}$.
The above problem is solved using accelerated proximal gradient descent \cite{FISTA:Beck2009}, where the proximal operator for the squared $k$-support norm $\text{prox}_{\lambda \left(\|\cdot\|_k^{sp}\right)^2}(\cdot)$ is given by \cite{ksupport_nips2012}. 
\subsection{Algorithm}
Algorithm 1 summarizes the procedure to optimize the objective function \eqref{eq:objective}.
We set the ADMM parameter $\rho$ to 2 and consider that the ADMM for updating \textbf{U} converges if  $\|\textbf{r}^{t+1}\|_F\leq 0.01$ and $\|\textbf{s}^{t+1}\|_F\leq 0.01$.
\section{Theoretical analysis}
In this section, we provide an upper bound on excess error of the proposed method based on the previous work from Maurer et al. \cite{MTL:low_rank_bound_Maurer2016}. 

Suppose $\mu_1,\ldots,\mu_T$ be probability measures on $\mathbb{R}^D \times \mathbb{R}$. 
Then, the input matrix $\textbf{X}_j$ and the output vector $\textbf{y}_j$ are drawn from the probability measure $\mu_j$ with $N_j=N$. 
We express $\bar{\textbf{X}}=[\textbf{X}_1,\ldots,\textbf{X}_T]$.

The optimization problem (5) is reformulated as follows: 
\begin{equation}
\label{prob:trans}
\min_{\textbf{U} \in \mathcal{H}, \textbf{v}_j \in \mathcal{F}} 
\frac{1}{NT} \sum_{j=1}^T L'\left( \textbf{y}_j,\textbf{X}_j\textbf{U}\textbf{v}_j\right), 
\end{equation}
where $\mathcal{H}= \Big\{ \textbf{x} \in \mathbb{R}^D \to \left(\textbf{u}_1^T\textbf{x},\ldots,\textbf{u}_K^T\textbf{x}\right)\in \mathbb{R}^K:\textbf{u}_1,\ldots,\textbf{u}_K \in \mathbb{R}^D, \sum_{r=1}^K \|\textbf{u}_r\|_{1} \leq \alpha_1, \sum_{i=1}^D \|\textbf{u}^i\|_{1,\infty}\leq \alpha_2 \Big\}$, and  
$\mathcal{F} = 
\Big\{ \textbf{z}\in \mathbb{R}^K \to \textbf{v}^T\textbf{z} \in \mathbb{R}:
\textbf{v} \in \mathbb{R}^K, 
\left(\|\textbf{v} \|_{sp}^{k} \right)^2 \leq \beta^2
\Big\}$, and $L'$ is the scaled loss function in $[0,1]$.
We are interested in the expected error given by
\begin{displaymath}
\varepsilon \left(\textbf{U},\textbf{v}_1,\ldots,\textbf{v} _K\right) := \frac{1}{T} \sum_{j=1}^{T} \mathbb{E}_{\left(\textbf{X}_j,\textbf{y}_j\right) \sim \mu_j}
L'\left(\textbf{y}_j,\textbf{X}_j\textbf{U}\textbf{v}_j\right).
\end{displaymath}
Let $\hat{\textbf{U}}$ and $\hat{\textbf{v}}_1,\ldots,\hat{\textbf{v}}_K$ be the optimal solution of the optimization problem \eqref{prob:trans}, then we have the following theorem
\begin{theorem}
\normalfont
\textbf{(Upper bound on excess error).} \textit{If} $\alpha_1^2\leq K$, \textit{with probability at least  1- $\delta$ in} $\bar{\textbf{X}}$ \textit{the excess error is bounded by}
\begin{gather*}
\varepsilon \left(\hat{\textbf{U}},\hat{\textbf{v}}_1,\ldots,\hat{\textbf{v}}_K \right) -
\min_{\textbf{U} \in \mathcal{H}, \textbf{v}_j\in \mathcal{F}} \varepsilon \left(\textbf{U},\textbf{v}_1,\ldots,\textbf{v}_K \right)
\nonumber \\ 
\leq c_1\beta K\sqrt{\frac{\|\hat{C}\left(\bar{\textbf{X}}\right)\|_1}{NT}} 
+ c_2\beta \sqrt{\frac{K\|\hat{C}\left( \bar{\textbf{X}}\right)\|_{\infty}}{N}} + \sqrt{\frac{8\ln(2/\delta)}{NT}},
\end{gather*}
\textit{where } 
$\|\hat{C}\left(\bar{\textbf{X}}\right)\|_1=\frac{1}{T}\sum_{j=1}^T tr\left(\hat{\Sigma}(\textbf{X}_j)\right)$,
$\|\hat{C}\left( \bar{\textbf{X}}\right)\|_{\infty} = \frac{1}{T}\sum_{j=1}^T \lambda_{max}\\
\left(\hat{\Sigma}(\textbf{X}_j)\right)$, 
$\hat{\Sigma}(\textbf{X}_j)$ \textit{is the empirical covariance of input data for the} $j$\textit{-th task,} $\lambda_{max}(\cdot)$ \textit{is the largest eigenvalue, and } $c_1$ \textit{and} $c_2$ \textit{are universal constants}.  
\end{theorem} 
\begin{proof}
We can show that 
$\sum_{r=1}^K\|\hat{\textbf{u}}_r\|_2^2 \leq \sum_{r=1}^K\|\hat{\textbf{u}}_r\|_1^2 \leq \alpha_1^2 \leq K$
and $\|\hat{\textbf{v}}\|_2^2 \leq
\left(\| \hat{\textbf{v}}\|_{k}^{sp}\right)^2 \leq \beta^2$ considering $\hat{\textbf{U}} \in \mathcal{H}$ and $\hat{\textbf{v}}_j \in \mathcal{F}$. 
Then, the optimization problem \eqref{prob:trans} satisfies the conditions on Lemma 3 and Theorem 4 in \cite{MTL:low_rank_bound_Maurer2016} and the result follows. 
\end{proof}
The above theorem shows the roles of the hyper-parameters. 
The constraint parameters $\alpha_1$ and $\beta$ should be low enough to satisfy $\alpha_1^2 \leq K$, and produce a tighter bound. Thus, the corresponding regularization parameters $\gamma_1$ and $\mu$ should be large enough to fulfill the above condition and tighten the bound. 
\section{Experiment}
In this section, we present experiments conducted to evaluate the effectiveness of our proposed method. We compare our proposed methods with the following benchmark methods:
\begin{itemize}
\item \textbf{LASSO method}: This single-task learning method learns a sparse prediction model for each task independently. 
\item \textbf{L1+trace norm \cite{MTL:low_rank_l1_trace_Richard2012}}: This MTL method simultaneously achieves a low-rank structure and variable selection by penalizing both the nuclear norm and the $\ell_1$ norm of the coefficient matrix.  
\item \textbf{Multiplicative multi-task feature learning (MMTFL) \\ \cite{MTL:MMTFL_Wang2016}}: This MTL method factorizes a coefficient matrix as the product of full rank matrices to select the relevant input variables. In this paper, we set $p=1$ and $k=2$.
\item \textbf{Group overlap multi-task learning (GO-MTL) \cite{MTL:low_rank_group_Kumar2012}}: This MTL method factorizes a coefficient matrix as the product of low-rank matrices and learn an overlapping group structure among tasks by imposing sparsity on the  weighting vectors.
\end{itemize}
The hyper-parameters of all methods are selected by minimizing the error from an inner 10-fold cross validation step or a validation set. 
To reduce the computational complexity of the proposed method, we set the third regularization parameter $\mu$ to be equal to the first regularization parameter $\gamma_1$.
The regularization parameters of all methods are selected from the search grid $\{2^{-10},\ldots,2^{3}\}$. For GO-MTL and VSTG-MTL, the number of latent bases $K$ is selected from the search grid $\{3,5,7,9,11,13\}$.
For the synthetic datasets, the value of $k$ is set to $1$ (\textbf{VSTG-MTL $k$=1}), which is equivalent to the squared $\ell_1$ norm, or $3$ (\textbf{VSTG-MTL $k$=3}) to identify the effectiveness of the $k$-support norm for correlated features. 
In the real datasets, it is selected from $\{1,3,5,7\}$ (\textbf{VSTG-MTL $k$=opt}).
The Matlab implementation of the proposed method is available at the following URL: \url{https://github.com/JunYongJeong/VSTG-MTL}.

The evaluation measurements approach used are the root mean squared error (\textbf{RMSE}) for a regression problem and the error rate (\textbf{ER}) for a classification problem.
For synthetic datasets, we also compute the relative estimation error (\textbf{REE}) $\|\textbf{W}^*-\hat{\textbf{W}}\|_F/\|\textbf{W}^*\|_F$, 
where $\textbf{W}^*$ is the true coefficient matrix and $\hat{\textbf{W}}$ is the estimated one.
We repeat the experiments 10 times and compute the mean and standard deviation of the evaluation measurement.
We also perform a Wilcoxon signed rank test with $\alpha=0.05$, which is a non-parametric paired t-test, to find the best model statistically. The statistically best models are highlighted in bold  in Tables 1 and 2. 
\subsection{Synthetic Datasets}

We generate the following four synthetic datasets.
We use 25-dimensional variables ($D=25$) and 20 tasks ($T=20$). 
For the $j$-th task, we generate 50 training observations and 100 test observations from $\textbf{x}_j^n \sim \mathcal{N}(\textbf{0},\textbf{I}_{D})$ and $y_j^n=\textbf{w}_j^T\textbf{x}_j^n+ \mathcal{N}(0,1)$. A true coefficient matrix $\textbf{W}^*=[\textbf{w}_1^*,\ldots,\textbf{w}_T^*]$ has a low-rank structure $K:=rank(\textbf{W})=5$ and is estimated by $\textbf{U}\textbf{V}$, where $\textbf{U}\in \mathbb{R}^{D\times K}$ and $\textbf{V} \in \mathbb{R}^{K\times T}$. Each synthetic dataset differs on the structure of the two matrices $\textbf{U}$ and $\textbf{V}$. 
\begin{itemize}
\item {\textbf{Syn1. Orthogonal features and disjoint task groups}}
\\: For $r=1,\ldots,K$, the latent basis $\textbf{u}_r$ only has non-zero values from the $\left(4r-3\right)$-th to the $4r$-th components.
The non-zero values are generated through a normal distribution with mean 1.0 and std 0.25.
Similarly, the weighting vectors $\textbf{v}_{4r-3}, \ldots,\textbf{v}_{4r}$ only have nonzero values on the $r$-th component.
The nonzero values are generated through a uniform distribution from 1 to 1.5. 
Thus, the last five variables are irrelevant.
The latent bases $\textbf{u}_r$, $r=1,\ldots,K$, as well as the corresponding features, are orthogonal to each other.
Each latent basis $\textbf{u}_r$ forms a disjoint group, where each group consists of four variables and tasks. 
\item {\textbf{Syn2. Orthogonal features and overlapping task groups}}
\\: The variable-latent matrix \textbf{U} is generated by the same procedure as that shown in Syn1. 
For $r=1,\ldots,K-1$, the weighting vectors $\textbf{v}_{4r-3},\ldots,\textbf{v}_{4r}$ only have
nonzero values on the $r$-th and $(r+1)$-th components.
The last four weighting vectors $\textbf{v}_{4K-3},\ldots,\textbf{v}_{4K}$ only have the nonzero values on the $\left(K-1\right)$-th and $K$-th components.
The nonzero values are generated using the same uniform distribution as that used in Syn1.
Then, the last five variables are irrelevant and the features are still orthogonal.
The tasks have $K$ overlapping groups, where each group consists of four variables and five tasks. 
\item{\textbf{Syn3. Correlated features and disjoint task groups}}
\\: For $r=1,\ldots,K$, the latent basis $\textbf{u}_r$ only has nonzero values from the $\left(3r-2\right)$-th to the $\left( 3r+3 \right)$-th components.
The nonzero values are generated using the same normal distribution as that used in Syn1.
The latent-task matrix \textbf{V} is generated using the same procedure as that used in Syn1.
The last seven variables are irrelevant and the latent bases are not orthogonal, resulting in correlation among features.
The tasks have $K$ disjoint groups, where each group consists of six variables and four tasks.
\item {\textbf{Syn4. Correlated features and overlapping task groups}}
\\: The variable-latent matrix $\textbf{U}$ is generated using the same procedure as that used in Syn3. 
The latent-task matrix $\textbf{V}$ is generated using the same procedure as that used in Syn2.
The last seven input variables are irrelevant. Thus, the features are correlated and the tasks have $K$ overlapping groups, where each group consists of six variables and five tasks.
\end{itemize}
\begin{table*}[ht]
\centering
\caption{Results for the synthetic datasets showing the average RMSE and REE with 10 repetitions. The statistically best models are highlighted in bold.}  
\begin{tabular}{llccccccc}
\hline
Synthetic & Measure & LASSO & L1+trace & MMTFL & GO-MTL & VSTG-MTL $k=1$ & VSTG-MTL $k=3$
\\ \hline
Syn1 & RMSE & 1.4625 $\pm$ 0.1349 & 1.1585 $\pm$ 0.0180 & 1.1384 $\pm$ 0.0257 & 1.0935 $\pm$ 0.0185 & \textbf{1.0456 $\pm$ 0.0228} &  1.0766 $\pm$ 0.0176\\
          & REE & 0.4155 $\pm$ 0.0595 & 0.2249 $\pm$ 0.0200 & 0.2089 $\pm$ 0.0169 & 0.1737 $\pm$ 0.0165 & \textbf{0.1226 $\pm$ 0.0149} & 0.1536 $\pm$ 0.0128 \\ \hline
Syn2      & RMSE    & 1.6811 $\pm$ 0.1146 & 1.2639 $\pm$ 0.0418 & 1.2377 $\pm$ 0.0401 & 1.1509 $\pm$ 0.0267 & 1.1294 $\pm$ 0.0332 & 1.1314 $\pm$ 0.0275 \\ 
          & REE & 0.3703 $\pm$ 0.0441 & 0.2040 $\pm$ 0.0169 & 0.1921 $\pm$ 0.0152 & 0.1488 $\pm$ 0.0122 & 0.1365 $\pm$ 0.0135 & 0.1376 $\pm$ 0.0091 \\ \hline
Syn3 & RMSE & 1.5303 $\pm$ 0.0483 & 1.2244 $\pm$ 0.0320 & 1.1797 $\pm$ 0.0287 & 1.1129 $\pm$ 0.0250 & 1.1086 $\pm$ 0.0192 & \textbf{1.1020 $\pm$ 0.0226} \\
    & REE & 0.3801 $\pm$ 0.0328 & 0.2262 $\pm$ 0.0211 & 0.2001 $\pm$ 0.0168 & 0.1565 $\pm$ 0.0148 & 0.1538 $\pm$ 0.0123 & \textbf{0.1486 $\pm$ 0.0121} \\ \hline
Syn4      & RMSE    &  1.7380 $\pm$ 0.1032 & 1.2673 $\pm$ 0.0312 & 1.2271 $\pm$ 0.0309 &  1.1278 $\pm$ 0.0235 & 1.1139 $\pm$ 0.0236 & \textbf{1.1085 $\pm$ 0.0214}
\\          & REE     & 0.2729 $\pm$ 0.0365 & 0.1419 $\pm$ 0.0125 & 0.1302 $\pm$ 0.0111 & 0.0945 $\pm$ 0.0087 & 0.0895 $\pm$ 0.0094 & \textbf{0.0868 $\pm$ 0.0083}
\\ \hline
\end{tabular}
\end{table*}
\begin{figure*}[t]
\begin{subfigure}[t]{0.24\textwidth}
\includegraphics[width=\textwidth, height=1.6in]{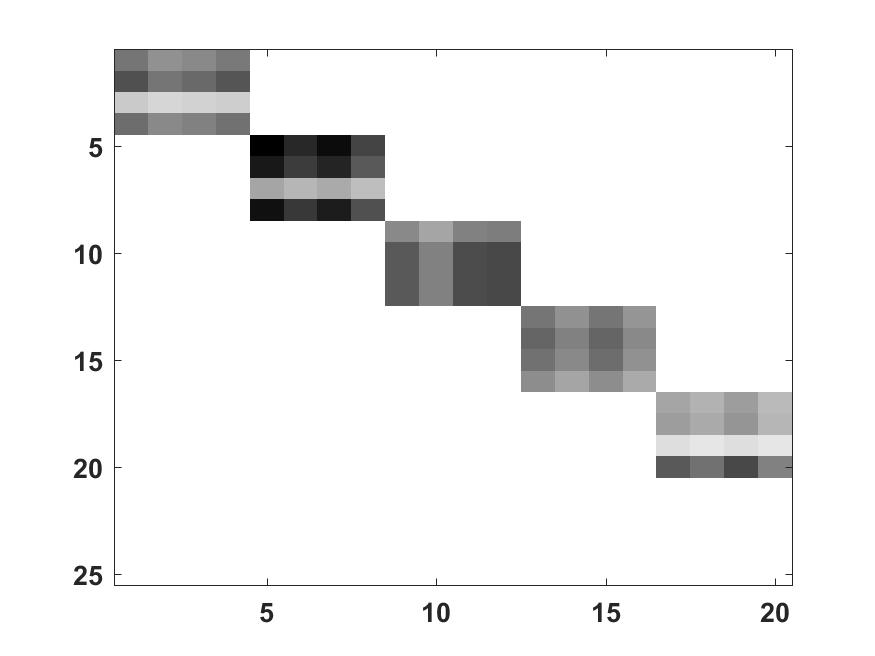}
\caption{Syn1 true}
\end{subfigure}
\begin{subfigure}[t]{0.24\textwidth}
\includegraphics[width=\textwidth, height=1.6in]{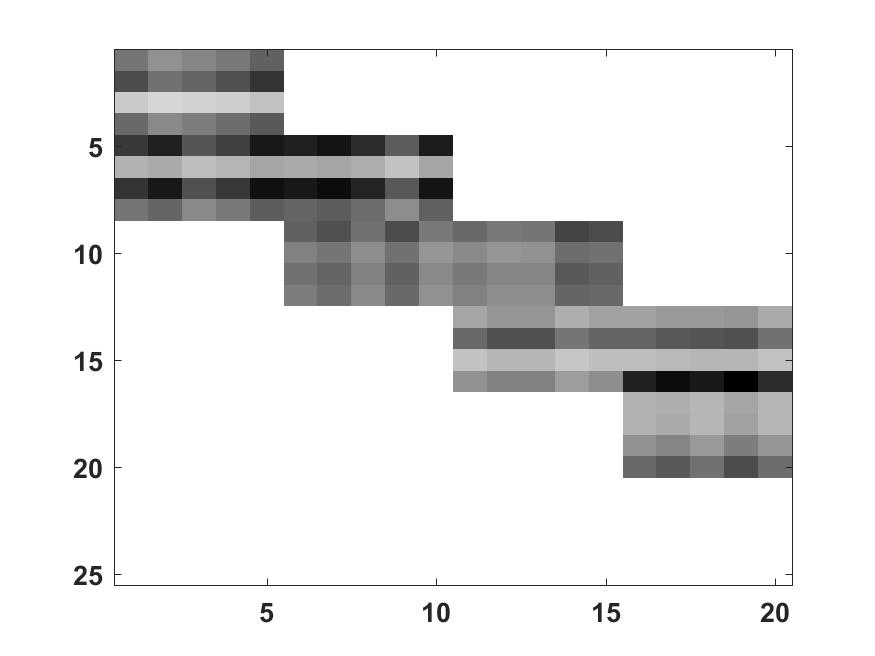}
\caption{Syn2 true}
\end{subfigure}
\begin{subfigure}[t]{0.24\textwidth}
\includegraphics[width=\textwidth, height=1.6in]{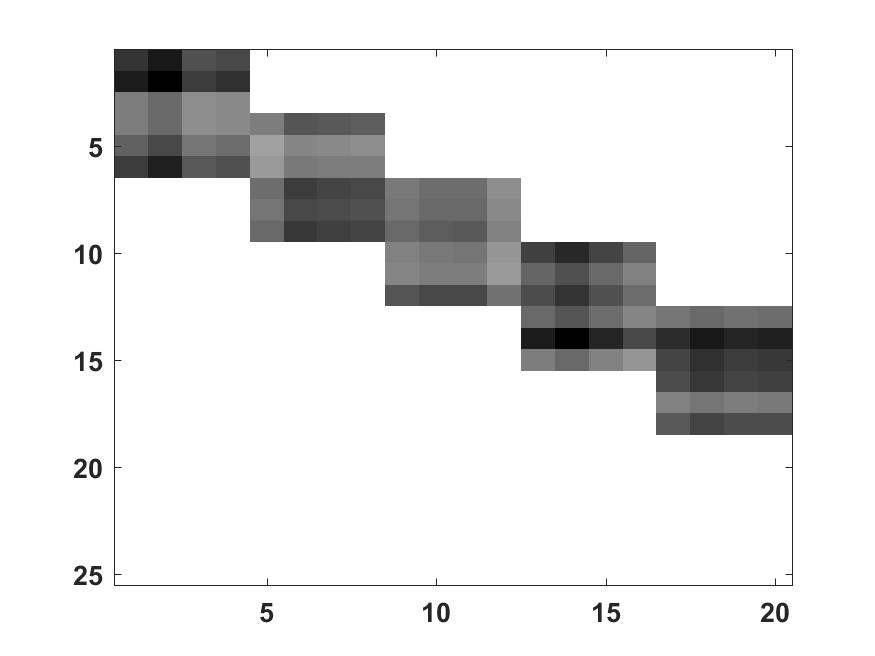}
\caption{Syn3 true}
\end{subfigure}
\begin{subfigure}[t]{0.24\textwidth}
\includegraphics[width=\textwidth, height=1.6in]{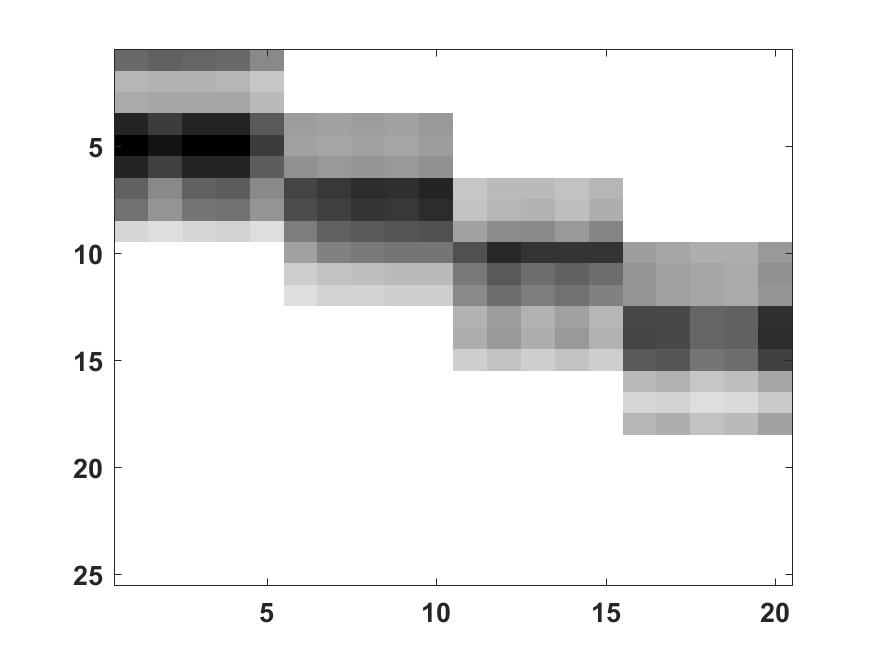}
\caption{Syn4 true}
\end{subfigure}
\\
\begin{subfigure}[t]{0.24\textwidth}
\includegraphics[width=\textwidth, height=1.6in]{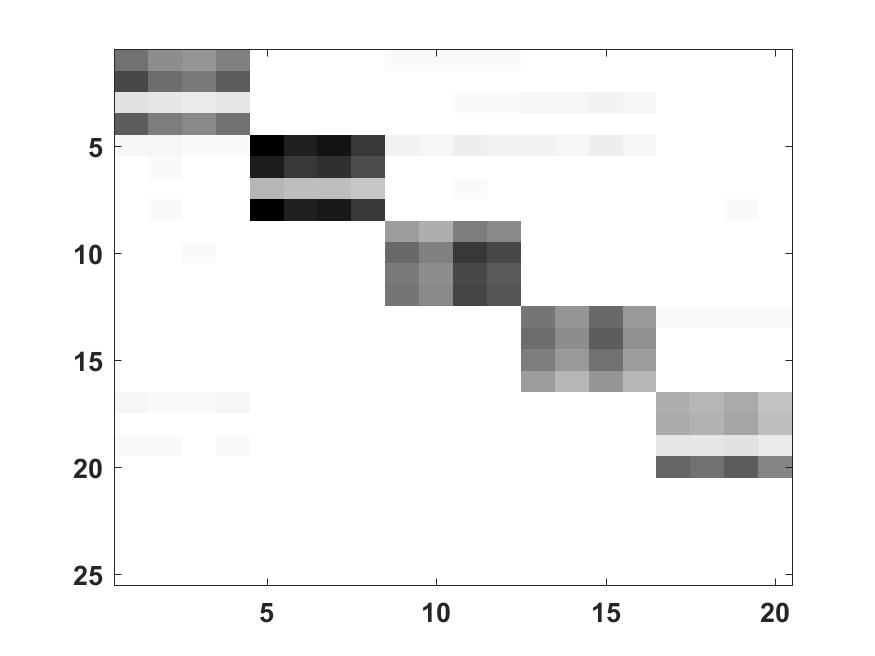}
\caption{Syn1 VSMTL $k$=1}
\end{subfigure}
\begin{subfigure}[t]{0.24\textwidth}
\includegraphics[width=\textwidth, height=1.6in]{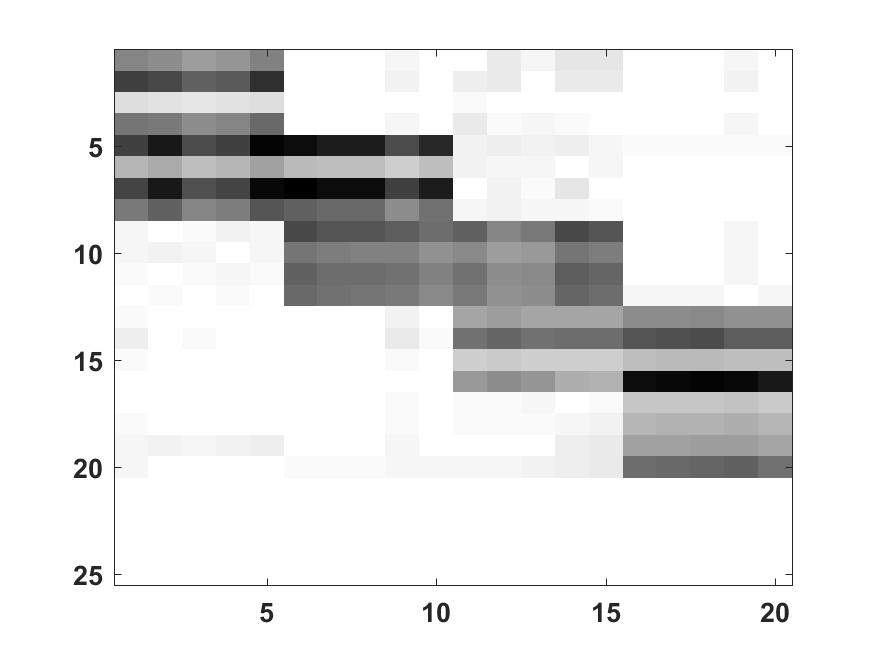}
\caption{Syn2 VSMTL $k$=1}
\end{subfigure}
\begin{subfigure}[t]{0.24\textwidth}
\includegraphics[width=\textwidth, height=1.6in]{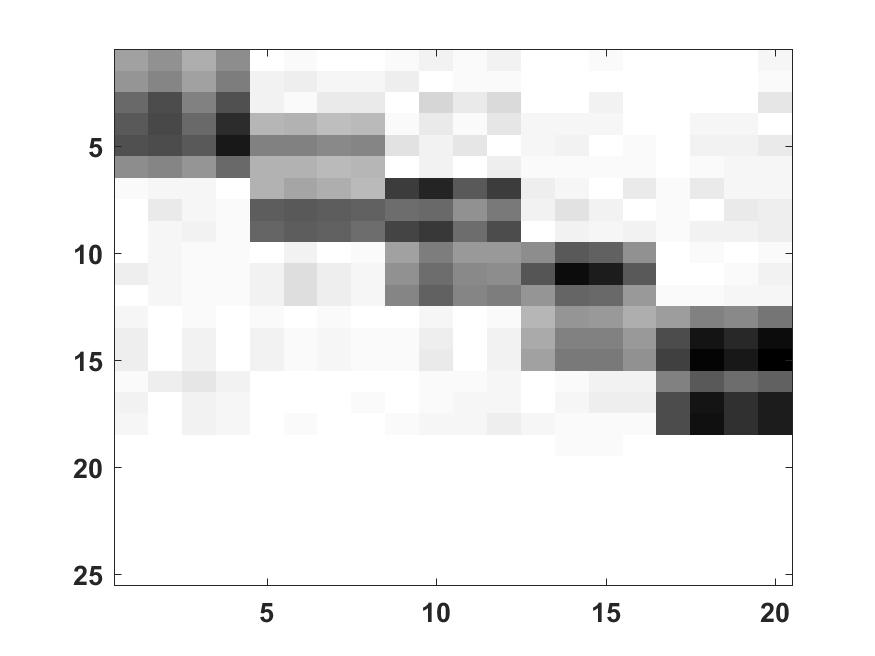}
\caption{Syn3 VSMTL $k$=3}
\end{subfigure}
\begin{subfigure}[t]{0.24\textwidth}
\includegraphics[width=\textwidth, height=1.6in]{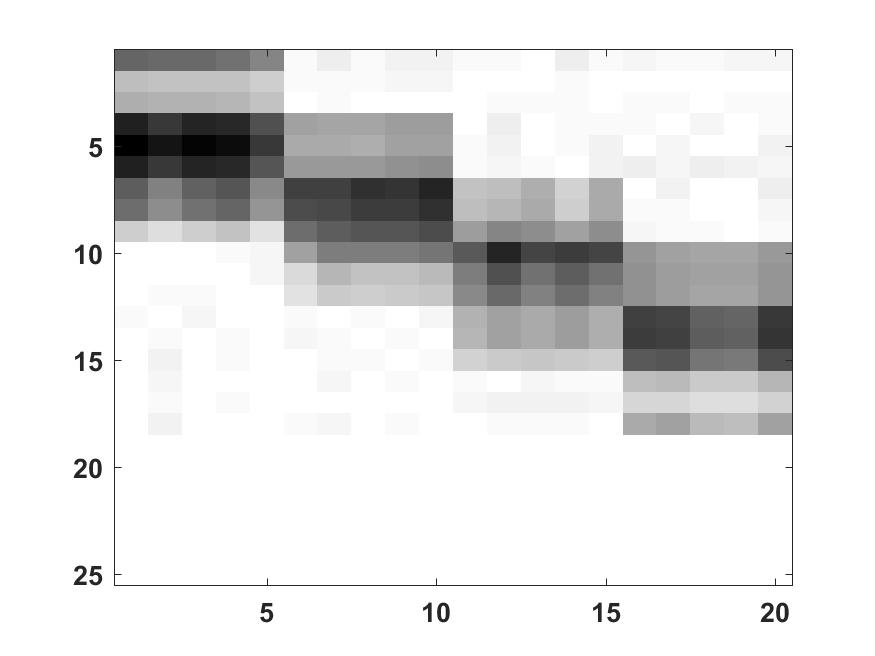}
\caption{Syn4 VSMTL $k$=3}
\end{subfigure}
\caption{True and estimated coefficient matrices by VSTG-MTL. 
The dark and white color entries indicate the large and zero values, respectively.}
\end{figure*}
Table 1 summarizes the experimental results for the four synthetic datasets in terms of RMSE and REE. 
For all the synthetic datasets, the MTL methods outperform the single-task learning method LASSO and VSTG-MTL exhibits the best performance.
Moreover, we can identify the effect of the $k$-support norm on the correlated features. 
On Syn1 and Syn2, where the latent bases $\textbf{u}_r$, $r=1,\ldots,K$ are orthogonal
, \textbf{VSTG-MTL $k$=1} outperforms  \textbf{VSTG-MTL $k$=3}.
This results indicates that the squared-$\ell_1$ norm penalty performs better than the squared $k$ support norm penalty with $k=3$ when the features are orthogonal. 
In contrast, on Syn3 and Syn4, where the latent bases $\textbf{u}_r$, $r=1,\ldots,K$ are not orthogonal,
\textbf{VSTG-MTL $k$=3} outperforms \textbf{VSTG-MTL $k$=1}. 
These results confirm to our premise that the $k$-support norm penalty can improve generalization performance more than the $\ell_1$ norm penalty when correlation exists.

The true coefficient matrix and estimated matrix using the proposed method are shown in Fig. 2, 
where the dark and white color entries indicate large and zero values, respectively.
VSTG-MTL can recover a group structure among tasks and exclude irrelevant variables.  
\subsection{Real Datasets}
We also evaluate the performance of  VSTG-MTL on the following five real datasets. 
After splitting the dataset into a training set and a test set, we transform the continuous input variables from the training set into $[-1,1]$ by dividing the maximums of their absolute values. Then, we divide the continuous input variables in the test set by using the same values as those in the training set. 
\begin{table*}[t]
\centering
\caption{Results form the real datasets for the RMSE and ER over 10 repetitions. The statistically best models are highlighted in bold.}
\label{tbl:real_datasets}
\begin{tabular}{llccccc}
\hline
Dataset & Measure &LASSO & L1+Trace & MMTFL & GO-MTL & VSTG-MTL $k$=opt
\\ \hline
School exam & RMSE & 12.0483 $\pm$ 0.1738 & 10.5041 $\pm$ 0.1432 & 10.1303 $\pm$ 0.1291  & 10.1924 $\pm$ 0.1331 & \textbf{9.9475 $\pm$ 0.1189} \\ 
Parkinson & & 2.9177 $\pm$ 0.0960 & 1.0481 $\pm$ 0.0243 & 1.1079 $\pm$ 0.0182 & 1.0231 $\pm$ 0.0285  & 1.0076 $\pm$ 0.0188 \\ 
Computer survey  &  &  2.3119 $\pm$ 0.3997
& 4.9493 $\pm$ 2.1592 & 1.7525 $\pm$ 0.1237 &  1.9067 $\pm$ 0.1864 & \textbf{1.6866 $\pm$ 0.1463}
\\ \hline
MNIST & ER &  13.0200 $\pm$ 0.7084 & 17.9800 $\pm$ 1.7574 & 12.6000 $\pm$ 0.8641 & 12.8400 $\pm$ 1.2989 & \textbf{11.7000 $\pm$ 1.4461}
\\ 
USPS & & 12.8800 $\pm$ 1.5061 & 16.0200 $\pm$ 1.2874 & 11.3600 $\pm$1.1462 & 12.9000 $\pm$ 1.0842 & 12.1800 $\pm$ 1.3547
\\ \hline
\end{tabular}
\end{table*}
\begin{itemize}
\item{\textbf{School exam dataset}\footnote{\url{http://ttic.uchicago.edu/~argyriou/code/index.html}
} \cite{dataset-student_exam}}: This multi-task regression dataset is obtained from the Inner London Education Authority.
It consists of an examination of 15362 students from 139 secondary schools in London during a three year period: 1985-1987. 
We have 139 tasks and 15362 observations, where each task and observation correspond to a prediction of the exam scores of a school and a student, respectively.
Each observation is represented by 3 continuous and 23 binary variables including school and student-specific attributes.
We follow the split procedure shown in \cite{MTL:low_rank:Argyriou2008}, resulting in a training set of 75\% observations and a test set of 25\% observations.
\item{\textbf{Parkinson's disease dataset} \footnote{\url{http://archive.ics.uci.edu/ml/datasets/Parkinsons+Telemonitoring}
} \cite{dataset-parkinson}}:
This multi-task regression dataset is obtained from biomedical voice measurements taken from 42 people with early-stage Parkinson's disease.
We have 42 tasks and 5875 observations, where each task and observation correspond to a prediction of the symptom score (motor UPDRS) for a patient and a record of a patient, respectively.
Each observation is represented by 19 continuous variables including age, gender, time interval, and voice measurements. 
We use 75\% of the observations as a training set and the remaining 25\% as a test set. 
\item {\textbf{Computer survey dataset} \footnote{\url{https://github.com/probml/pmtk3/tree/master/data/conjointAnalysisComputerBuyers}
} \cite{dataset-survey}}: This multi-output regression dataset is obtained from a survey of 190 ratings from people about their likelihood of purchasing each of the 20 different personal computers. 
We have 190 tasks and 20 observations shared for all tasks, where each task and observation correspond to a prediction of the integer ratings of a person on a scale of 0 to 10 and a computer.
Each observation is represented by 13 binary variables, including its specification.
We insert an additional variable to account for the bias term and use 75\% of the observations as a training set and the remaining 25\% as a test set. 
\item {\textbf{MNIST dataset}\footnote{\url{http://yann.lecun.com/exdb/mnist/}} }\cite{dataset:mnist}:
This multi-class classification dataset is obtained from 10 handwritten digits.
We have 10 tasks, 60,000 training observations and 10,000 test observations, where each task and observation correspond to a prediction of the digit and an image, respectively. 
Each observations is represented by $28\times 28$ variables and reduced to 64 dimensions using PCA.
Train, validation and test set are generated by randomly selecting 1,000 observations from the train set of and two sets of 500 observations from the test set, similar to the procedure of Kang et al. \cite{MTL:Low_rank_Kang2011}. 
\item{\textbf{USPS dataset}\footnote{\url{http://www.cad.zju.edu.cn/home/dengcai/Data/MLData.html}}} \cite{dataset:usps}:
This multi-class classification dataset is also obtained from the 10 handwritten digits.
We have 10 tasks, 7,291 training observations and 2,007 test observations, where each task and observation correspond to a prediction of the digit and an image, respectively.
Each observation is represented by $16\times 16$ variables and reduced to 87 dimensions using PCA. 
,We follow the same procedure of that used in the MNIST dataset to generate train, validation and test set, resulting in 1000, 500, and 500 observations, respectively. 
\end{itemize}
Table 2 summarizes the results for the five real datasets over 10 repetitions. \textbf{VSTG-MTL $k$=opt} outperforms the benchmark methods except the USPS dataset. 
This is especially true for the school exam dataset, the computer survey dataset, and the MNIST dataset, where the proposed method shows statistically significant improvements over the benchmark methods. 
\section{Conclusion}
This paper proposes a novel algorithm of VSTG-MTL, which simultaneously performs variable selection and learns an overlapping group structure among tasks.
VSTG-MTL factorizes a coefficient matrix into the product of low-rank matrices and impose sparsities on them while considering possible correlations. 
The resulting bi-convex constrained problem is transformed to a regularized problem that is solved by alternating optimization. 
We provide the upper bound on the excess risk of the proposed method. 
The experimental results show that the proposed VSTG-MTL method outperforms the benchmark methods on synthetic as well as real datasets.
\begin{acks}
This study was supported by National Research Foundation of Korea (NRF) grant funded by the Korean government (MEST) (No 2017R1A2B4005450)
\end{acks}

\bibliographystyle{ACM-Reference-Format}
\bibliography{reference} 

\end{document}